\documentclass[letterpaper, 10 pt, conference]{ieeeconf}  

\IEEEoverridecommandlockouts                              
\usepackage{graphics} 
\usepackage{epsfig} 
\usepackage{mathptmx} 
\usepackage{times} 
\usepackage{amsmath} 
\usepackage{amssymb}  
\usepackage{amsfonts}
\usepackage{color}
\usepackage{hyperref}
\usepackage{algorithmic,algorithm}
\usepackage{float}
\usepackage{caption}

\DeclareMathOperator*{\argmin}{arg\,min}
\DeclareMathOperator{\vp}{\varphi}

\newcommand{\bR}{\mathbb{R}}
\newcommand{\bN}{\mathbb{N}}
\newcommand{\cB}{\mathcal{B}}
\newcommand{\cF}{\mathcal{F}}
\newcommand{\E}{\mathbb{E}}
\newcommand{\cN}{\mathcal{N}}

\newcommand{\cD}{\mathcal{D}}
\newcommand{\cV}{\mathcal{V}}
\newcommand{\bZ}{\mathbb{Z}}
\newcommand{\bK}{\mathbb{K}}

\newcommand{\bJ}{\mathbb{J}}
\newcommand{\cP}{\mathcal{P}}
\newcommand{\bD}{\mathbb{D}}

\newcommand{\bL}{{\bf L}}
\newcommand{\ba}{{\bf a}}
\newcommand{\bb}{{\bf b}}
\newcommand{\bX}{{\bf X}}

\newtheorem{theorem}{Theorem}[section]
\newtheorem{assumption}{Assumption}[section]
\newtheorem{definition}{Definition}[section]
\newtheorem{example}{Example}[section]
\newtheorem{lemma}{Lemma}[section]
\newtheorem{corollary}{Corollary}

\newtheorem{claim}{Claim}

\title{\LARGE \bf
Learning the conditional law: signatures and conditional GANs in filtering and prediction of diffusion processes
}

\author{Fabian Germ and Marc Sabate-Vidales
\thanks{F. Germ and M. Sabate-Vidales are with School of Mathematics and Maxwell Institute for Mathematical Sciences, University of Edinburgh, James Clerk Maxwell Building, Kings Buildings, Peter Guthrie Tait Road Edinburgh, EH9 3FD, U.K.
        {\tt\small \{f.germ,m.sabate-vidales\}@sms.ed.ac.uk},
}
}
\begin{document}

\maketitle
\thispagestyle{empty}
\pagestyle{empty}

\begin{abstract}
We consider the filtering and prediction problem for a diffusion process. The signal and observation are modeled by stochastic differential equations (SDEs) driven by correlated Wiener processes. In classical estimation theory, measure-valued stochastic partial differential equations (SPDEs) are derived for the filtering and prediction measures. These equations can be hard to solve numerically. We provide an approximation algorithm using conditional generative adversarial networks (GANs) in combination with signatures, an object from rough path theory. The signature of a sufficiently smooth path determines the path completely. As a result, in some cases, GANs based on signatures have been shown to efficiently approximate the law of a stochastic process. 
For our algorithm we extend this method to sample from the conditional law, given noisy, partial observation. Our generator is constructed using neural differential equations (NDEs), relying on their universal approximator property. We show well-posedness in providing a rigorous mathematical framework. Numerical results show the efficiency of our algorithm.
\end{abstract}

\section{INTRODUCTION}
\label{sec intro}

In many applications the state of a system is not fully observable and instead only partial or noisy information is available, from which the state has to be estimated.
The development of estimation theory for deterministic and stochastic dynamical systems has received enormous attention over the past decades. A very common model are SDEs, consisting of a signal $(X_t)_{t\in [0,T]}$ and an observation $(Y_t)_{t\in [0,T]}$, the coefficients of which depend on $X$. Letting $\cF^Y_t$ denote the information available from $Y$ until time $t$, then it is known that under some conditions the conditional expectation $\E(X_t|\cF^Y_t)$ minimizes a mean square error. 
Sometimes it is desirable to, more generally, estimate
\begin{equation}
\label{equ cond exp}
\E(\vp(X_t)|\cF^Y_s)
\end{equation}
for times $s,t\in [0,T]$ and arbitrary $\vp\in C^\infty_0$.
This is referred to as smoothing, filtering and prediction if $t<s$, $s=t$ and  $t>s$ respectively.
This paper is concerned with the latter two.\newline
Filtering in particular is a heavily researched area and the literature on it is vast \cite{C2014}.
There are numerous methods of analyzing or deriving \eqref{equ cond exp}, \cite{BC2009}. If the signal and the observation are given by stochastic differential equations (SDEs) driven by Wiener processes, they model a partially observable diffusion process. Then it is possible to derive a measure valued SPDE for the time evolution of the normalized conditional distribution $P(X_t\in dx|\cF^Y_t)$, referred to as Kushner-Shiryaev equation \cite{CCC2014}. Such ``filtering equations'' have been thoroughly investigated for diffusion processes. In a similar way, SPDEs for the smoothing and prediction measures, respectively $P(X_t\in dx|\cF^Y_r)$ and $P(X_t\in dx|\cF^Y_s)$, $s<t<r$, can be obtained \cite{RL2018}. Under additional regularity assumptions,  it is possible to prove the existence and regularity of densities to these measures, see \cite{Y1977,KR1978, R1980, R1984} for early works on this and \cite{GG1,GG2} for a recent extension to systems with discontinuous noise.\newline
The filtering equations are numerically challenging to solve and different approaches have been developed to approximate their solutions \cite{BC2009}. One of the most common, the splitting method, seperates the right-hand side into a deterministic and a random operator and solves them separately \cite{GK2003, GK2003-2}.
This either requires additional assumptions on the spaces involved \cite{IR2000} or relatively high regularity of the coefficients \cite{GK2003, GK2003-2}. Moreover, numerical solvers for (S)PDEs often suffer from the curse of dimensionality. Recently the splitting-up method was combined with a neural net representation to overcome this \cite{CLO2022}.  The prediction density can then be obtained using the transition probability of the process $Z=(X,Y)$ \cite{R1980,RL2018}.\newline
While filtering and prediction theory are mainly developed in a stochastic setting, the techniques have also proven successful in observer design for determinisic systems, where the disturbance is a single, often continuous path \cite{RU1999,AGM2020,AGM2022}. \newline
Rough path theory is a young field, developed to treat differential equations driven by paths of low regularity that escape the scope of classical integration, see \cite{LCL2007} and the references therein. Signatures, as sequences of iterated integrals with similarity to the Taylor expansion, are an object arising in rough path theory which encode a surprising amount of information about the path. Early works by Chen \cite{C1954, C1958} show that the signature map, sending a path to its signature, is injective when restricting it to a certain class of paths and later results \cite{HL2010, BGLY2016} establish uniqueness up to tree-like equivalence for paths with bounded variation. These results motivated the use of signatures in machine learning \cite{CK2016} as paths obtained by interpolation of data can be uniquely characterized by their signature \cite{CO2018}.\newline
Generative Adversarial Networks (GANs) were first introduced in 2014~\cite{goodfellow2014generative} as a novel way to learn data distributions. In this setting, two neural networks, the \textit{generator} and the \textit{discriminator}, compete to generate and verify potentially new or fake data. In a time series context, there have been many recent contributions, especially attempting to learn conditional laws, \cite{marc-sig, yoon2019time}. \newline   
Neural differential equations (NDEs) have also had a sharp increase in their popularity and their applications in recent years with the release of several torch-based libraries \cite{KMFL2020}, \cite{chen2018neural}. In an NDE, a process is modeled as an ordinary/stochastic/controlled differential equation with a neural network (NN) as vector field(s). 
The relationship between deep NNs and continuous time models is clearly presented in \cite{chen2018neural} where the authors observe that a residual NN is a first-order Taylor expansion of an ODE with an NN as the vector field. Furthermore, they provide a memory-efficient training scheme (initially proposed by \cite{lecun1988theoretical} and \cite{pearlmutter1995gradient}) where the gradient of the loss in terms of the vector field's parameters is calculated as the solution of another backward neural ODE.  We refer the reader to \cite{kidger2022neural} for a comprehensive survey of NDEs and their different training methods. 

This paper presents a first signature-GANs based estimator for the conditional and prediction law of a diffusion process. More precisely, our model learns the family of conditional probability measures 
$$
P(X_t\in dx|\{Y_r:r\in [0,s]\}),\quad s\leq t,
$$  
by using the universal approximation property of neural differential equations (NDEs). We prove that the estimator is well-posed and give explicit forms of the neural nets involved. Numerical results show the efficiency of our method. In a subsequent paper we will provide a proof of convergence as well as error bounds for the learned conditional distributions to the true ones.

In section 2 we state conditions under which the filtering and prediction measures admit densities, required for well-posedness of our estimator. Section 3 introduces the signature and collects useful results. In section 4 we present our estimator and show well-posedness. Numerical results are shown in section 5.

\textbf{Notation.} Throughout the paper we use the following notation. We fix a $T>0$ and consider the time interval $[0,T]$. We denote by $\cD\subset [0,T]$ a finite set of points including $0$ and $T$. We denote by $C_\cD=C_\cD([0,T],\bR^d)$ the set of continuous functions that are linear between the points in $\cD$. If $(X_t)_{t\in [0,T]}$ is a stochastic process with contiunous sample paths, then we denote by $(\hat{X})_{t\in [0,T]}$ the process constructed by interpolating the points $X_{t_i}$, where $t_i\in\cD$. Moreover, for any process $(X_t)_{t\in [0,T]}$ the notation $\bar{X}$ means the time-augmented path $(t,X_t)_{t\in [0,T]}$. For $p\geq 1$ we denote by $\cV^p([0,T],\bR^d)$ the space of $\bR^d$-valued continuous paths with finite $p$-variation \cite{LCL2007}. On it, we consider the norm $\|X\|_{\cV^p}:=\|X\|_{BV,p}+\|X\|_\infty$, with $\|\cdot\|_{BV,p}$ denoting the $p$-variation and $\|X\|_\infty$ the essential supremum.


\section{The filtering and prediction densities}
Let $(\Omega,(\cF_t)_{t\in [0,T]},P)$ be a complete filtered probability space.
Consider the signal-observation system
\begin{equation}
\label{equ SDE}
    \begin{split}
        dX_t &= b(t,X_t,Y_t)\,dt + \sigma(t,X_t,Y_t)\,dW_t + \rho(t,X_t,Y_t)\,dV_t\\
        dY_t &= h(t,X_t,Y_t)\,dt + dW_t,
    \end{split}
\end{equation}
where $(X_t)_{t\in [0,T]}$ and $(Y_t)_{t\in [0,T]}$ are $\bR^{d}$ and $\bR^{d'}$-valued respectively, $b$, $h$, $\sigma$ and $\rho$ are $\cB(\bR^{1+d+d'})$-measurable functions taking values in $\bR^{d}$, $\bR^{d'}$, $\bR^{d\times d'}$ and $\bR^{d\times d''}$ respectively, and where $(W_t,V_t)_{t\in [0,T]}$ is a $d' + d''$-valued $\cF_t$-Wiener process.

\begin{assumption}
\label{assumption SDE}
(i) The initial condition $Z_0 = (X_0,Y_0)$ is  $\cF_0$-measurable, independent of $(W_t,V_t)_{t\in [0,T]}$ and satisfies $\E|Z_0|^{2}<\infty$.\newline
(ii) There are constants $K_0,K_1\geq 0$ such that for all $z\in \bR^{d+d'}$, $t\in [0,T]$ we have $|h(t,z)|\leq K_0$ and
$$
|b(t,z)|  + |\sigma(t,z)| 
+ |\rho(t,z)|\leq K_0 + K_1|z|.
$$
(iii) There exists $L\geq 0$ such that for $z_1,z_2\in\bR^{d+d'}$, $t\in [0,T]$,
$$
|b(t,z_1)-b(t,z_2)| + |h(t,z_1)-h(t,z_2)| + |\sigma(t,z_1)-\sigma(t,z_2)|
$$
$$
 + |\rho(t,z_1)-\rho(t,z_2)| +|(\sigma\cdot h)(t,z_1)-(\sigma\cdot h)(t,z_2) | \leq L|z_1-z_2|.
$$
\end{assumption}

If Assumption \ref{assumption SDE} is satisfied, then by a well-known theorem by It\^o we know that there exists a unique $\cF_t$-adapted solution  $Z=(Z_t)_{t\in [0,T]}=((X_t,Y_t))_{t\in [0,T]}$ to \eqref{equ SDE} such that $Z_0 = (X_0,Y_0)$ almost surely, $Z$ has continuous sample paths almost surely and $\E\sup_{t\in [0,T]}|Z_t|^2<\infty$.  Moreover, it covers a wide range of applications, as often more regularity than Lipschitzness is used in SDE models.

Let for $t\in [0,T]$
$$
\cF^Y_t:=\sigma\big(\{Y_r:r\in [0,t] \}\big)\vee \cN
$$
be the filtration generated by $Y$ and completed by the zero sets $\cN$. 
The goal of filtering and prediction is to derive and analyze the quantities 
$$
\E(\vp(X_t)|\cF^Y_t)\quad\text{and}\quad \E(\vp(X_t)|\cF^Y_s),\quad s\leq t,
$$
respectively, for $t\in [0,T]$ and $\vp\in C_0^\infty(\bR^{d})$. Fix an $s\in [0,T]$. In many applications it is desirable to obtain real valued density processes,  $(\pi_t)_{t\in [0,T]}$ and $(\pi_{t,s})_{t\in [0,T]}$, such that for each $s,t\in [0,T]$, $t\geq s$,
$$
\pi_t = \frac{P(X_t\in dx|\cF^Y_t)}{dx}\quad\text{and}\quad
\pi_{t,s} = \frac{P(X_t\in dx|\cF^Y_s)}{dx}
$$
almost surely and hence for all $\vp\in C_0^\infty$ and $s,t\in [0,T]$, 
$$
\E(\vp(X_t)|\cF^Y_s) = \int_{\bR^{d}}\vp(x)\pi_{t,s}(x)\,dx,
$$
almost surely, where $\pi_{s,s}=\pi_s$.

The following result is well-known. For a proof and more details we refer to \cite{BC2009, C2014, CCC2014} for classical literature, or \cite{GG1} and \cite{GG2} for a recent generalization to a wider class of systems.

\begin{theorem}
\label{theorem filtering density}
Let Assumption \ref{assumption SDE} (i) \& (ii) hold. Then there exists an $\cF^Y_t$-adapted measure valued process $P_t(dx)=P(X_t\in dx|\cF^Y_t)$ such that almost surely for each $t\in [0,T]$ and 
$\vp\in C_0^\infty (\bR^d)$,
$$
\E(\vp(X_t)|\cF^Y_t) = \int_{\bR^d}\vp(x)\,P_t(dx).
$$
If moreover Assumption \ref{assumption SDE} (iii) holds  and for $\pi_0=P(X_0\in dx|\cF^Y_0)/dx$ we have $\E|\pi_0|_{L_2}^2<\infty$, then there exists an $L_2$-valued weakly continuous process $(\pi_t)_{t\in [0,T]}$ such that for each $t\in [0,T]$ almost surely 
$$
\pi_t = P(X_t\in dx|\cF^Y_t)/dx.
$$
\end{theorem}

The following relates the prediction density to the filtering density, \cite{RL2018}. For that purpose, denote by $p(t_1,z;t_0,z_0)$ the transition probability for the process $(Z_t)_{t\in [0,T]}=((X_t,Y_t))_{t\in [0,T]}$, that is, for $t_0,t_1\in [0,T]$, $t_1\geq t_0$ and $z_0\in \bR^{d+d'}$, for every $B\in \cB(\bR^{d+d'})$,
$$
P(Z_{t_1}\in B| Z_{t_0} = z_0) = \int_{B}p(t_1,z;t_0,z_0)\,dz.
$$

\begin{theorem}
\label{th:prediction}
Let Assumption \ref{assumption SDE} hold and let $\E|\pi_0|_{L_2}^2<\infty$.
Assume the process $Z$ has the transition density $p=p(t_1,z;t_0,z_0)$, let $(\pi_t)_{t\in [0,T]}$ be the filtering density from Theorem \ref{theorem filtering density} and fix $s\in [0,T]$. Then there exists a measure valued process $(P_{t,s})_{t\in [s,T]}$ such that \newline 
(i) for each $t\in [s,T]$, $P_{t,s}$ is the regular conditional distribution of $X_t$ given $\cF_s^Y$ and such that \newline
(ii) for each $t\in [s,T]$, $P_{t,s}$ has the Radon-Nikodym derivative
$$
\pi_{t,s}(x) = \int_{\bR^{d}}\int_{\bR^{d'}} p(t,x,y;s,x',Y_s)\pi_s(x')\,dx'dy.
$$
\end{theorem}
For fixed $s\in [0,T]$ we call $(\pi_{t,s})_{t\in [s,T]}$ the prediction density of $(X_t)_{t\in [s,T]}$ given $\cF_s^Y$. Indeed, an immediate calculation shows that for each $\vp\in C_0^\infty$, $t\in [s,T]$, almost surely
$$
\E(\vp(X_t)|\cF_s^Y) = \int_{\bR^{d}}\vp(x)\pi_{t,s}(x)\,dx.
$$
Henceforth we assume for the conditions of Theorem \ref{th:prediction} to hold.


\section{Signatures and elements from rough path theory}

In this section we collect some objects and properties that will be used later. If not mentioned otherwise, the reader is referred to \cite{LCL2007} and the references therein, as well as to \cite{CK2016}, for the use of signatures in machine learning.

It is first necessary to introduce the space of formal series of tensors, which is the space signatures live in. For simplicity, we restrict ourselves to tensors over $\bR^d$. We denote by $(\bR^d)^{\otimes n}$ the usual space of tensors over $\bR^d$ of order $n\geq 0$.

\begin{definition} 
(i) The space of formal series of tensors of $\bR^d$, denoted by $T((\bR^d))$, is defined as space of sequences,
$$
T((\bR^d)):=\{\ba=(a_0,a_1,a_2,\dots):a_n\in (\bR^d)^{\otimes n}, n\in\bN \}.
$$
For two elements $\ba = (a_0,a_1,\dots)$ and $\bb = (b_0,b_1,\dots)$ we can define an addition and a product by
$$
\ba + \bb = (a_0+b_0,a_1+b_1,\dots),\quad \ba\otimes \bb = (c_0,c_1,\dots),
$$
where for each $n\in \bN_0$, with the usual (finite-dimensional) tensor product $\otimes$, $
c_n = \sum_{k=0}^n a_k\otimes b_{n-k}.
$\newline
(ii) Let $N\in \bN$ and define $B_N=\{\ba\in T((\bR^d)): a_0 = \cdots = a_N = 0\}.$ Then the truncated tensor algebra of order $N$ is the quotient algebra
$
T^N(\bR^d) = T((\bR^d))/B_N,
$
with the canonical homomorphism $\mathfrak{p}_N:T((\bR^d))\rightarrow T^N(\bR^d)$.
\end{definition}
We can naturally identify $T^N(\bR^d)$ with $\bR\oplus\bR^d\oplus\cdots\oplus(\bR^d)^{\otimes N}$.
Now we can introduce the (truncated) signature.

\begin{definition}
Let $X:[0,T]\rightarrow \bR^d$ be a path of finite variation and for $s,t\in [0,T]$, $n\in \bN$ define the iterated integral
$$
X_{s,t}^{(n)}:=\underset{s<s_1<\cdots<s_n<t}{\int\cdots\int}dX_{s_1}\otimes\cdots\otimes dX_{s_n}.
$$
Then the signature of $X$ over $(s,t)\subset [0,T]$ is
$$
\bX_{s,t} = (1,X^{(1)}_{s,t}, X^{(2)}_{s,t},\dots)\in T((\bR^d)).
$$
Similarly, the truncated signature is
$$
\bX_{s,t}^N = (1,X^{(1)}_{s,t}, \dots,X^{(N)}_{s,t})\in T^N(\bR^d).
$$
\end{definition}

\begin{example}
Consider $X_t=t$ on $[0,T]$. Then
$$
\bX_{0,T}=(1,X_T-X_0,\frac{(X_T-X_0)^2}{2!},\frac{(X_T-X_0)^3}{3!},\dots).
$$
\end{example}

Though signature captures deep geometric properties of a path, it does not necessarily characterise the path completely. It was shown that for continuous paths of bounded variation, the signature determines the path up to tree like equivalence \cite{HL2010}. A sufficient result for the present case is the following, Theorem 2.29 in \cite{LCL2007}.

\begin{theorem}
Among all paths with bounded variation sharing  the same signature, there exists a path with minimal length, which is unique up to reparametrization.
\end{theorem}
For linearly interpolated data points this in particular means the following. Recall the notation introduced in section \ref{sec intro}.
\begin{corollary}
\label{cor:signature uniqueness}
For $X,Y\in C_{\cD}$, $\bX = {\bf Y}$ only if $X=Y$.
\end{corollary}

It is clear that for a basis $(e_1,\dots,e_d)$ of  $\bR^d$ and its dual basis $(e_1^*,\dots,e_d^*)$ of $(\bR^d)^*$, the elements $(e_I=e_{i_1}\otimes\cdots\otimes e_{i_n})_{I = \{i_1,\dots,i_n\}\subset \{1,\dots,d\}^n}$ form a basis of $(\bR^d)^{\otimes n}$, just as the elements $(e_I^*=e^*_{i_1}\otimes\cdots\otimes e^*_{i_n})_{I = \{i_1,\dots,i_n\}\subset \{1,\dots,d\}^n}$ form a basis of $((\bR^d)^*)^{\otimes n}$.
Recall that we can canonically identify $((\bR^d)^*)^{\otimes n}$ with $((\bR^d)^{\otimes n})^*$. Thus we have a linear mapping $((\bR^d)^*)^{\otimes n}\rightarrow (T((\bR^d)))^*$ by the relation
$$
e^*_I(\ba)  = e^*_I(\mathfrak{p}_n(\ba)) = a_{i_1,\dots,i_n},
$$
which is the coefficient in front of the basis vector $e_I$ in $\ba$. In this way we get a linear mapping \cite{LCL2007}
$$
T((\bR^d)^*) = \bigoplus_{n=0}^\infty ((\bR^d)^*)^{\otimes n} \rightarrow (T((\bR^d)))^*.
$$
Thus also, for a path $X$ and its signature $\bX_{s,t}$, by linearity
$$
e^*_I(\bX) = \underset{s<s_1<\cdots<s_n<t}{\int\cdots\int}e_{i_1}^*(dX_{s_1})\otimes\cdots\otimes e^*_{i_n}(dX_{s_n}).
$$
\begin{definition}
Let $(\Omega,\cF,P)$ be a probability space, $(X_t)_{t\in [0,T]}$ an $\bR^d$-valued stochastic process and $\bX$ its signature. If $\E(\bX)<\infty$, then it is the expected signature of $X$. 
\end{definition}

The following is a very useful result \cite{LL2016}, which we rely on in our algorithm.

\begin{theorem}
\label{th:functionals}
Let $p\geq 1$, let $K\subset S(\cV^p([0,T],\bR^d))$ be compact and let $f:K\rightarrow \bR$ be continuous. Then for every $\varepsilon>0$ there exists a linear functional $\bL\in (T((\bR^d)))^*$ such that for all $\ba\in K$ we have 
$$
|f(\ba)-\bL\ba|\leq\varepsilon.
$$
\end{theorem}

While we do not go into details on the choice of a norm $\|\cdot\|$ on $T((\bR^d))$ (see \cite{LCL2007} for instance), we define, for functionals $f\in (T((\bR^d)))^*$,
$$
\|f\|_{Lip,1}=\sup_{\ba\neq\bb}\frac{|f(\ba)-f(\bb)|}{\|\ba-\bb\|}.
$$


\section{The Sig-Wasserstein-GAN predictor}

In the following we provide a mathematical framework for our approximation method. First we give precise meaning and forms to the NNs in the GAN used in our model. Then we outline the use of the Wasserstein distance on the signature space \cite{marc-sig, ni2020conditional} using expected signatures \cite{CO2018}.

\subsection{Well-posedness of the estimator}

\label{subsec: estimator}

We build the estimator as the composition of two neural differential equations (NDEs) \cite{kidger2022neural}. The first NDE encodes the information carried by the filtration $(\mathcal F_r^Y)_{0\leq r\leq s}$, whilst the second NDE is carefully designed so that its vector field parametrizes the rate of change of the mean prospective transition, see Lemma \ref{lm:diffeomorphisms} below.

\textbf{Estimator equations.} Let $\theta = (\theta_{1,0}, \theta_1, \theta_{2,0}, \theta_2)\in \Theta:=\mathbb R^{p_{1,0} + p_1 + p_{2,0} + p_2}$, for some $p_{1,0},p_1,p_{2,0},p_2\in\bN$ which we refer to as (learning) parameter and parameter space respectively.
Fix $k\in\bN$ and let $\bK:=\bR^k$ and $\bZ:=\bR^{d}$ be latent and sampling space respectively\footnote{In some applications it is useful to choose $k$ larger than $d$, as often a higher dimensional latent space proves to be more efficient in approximations.}. Let $z\sim \cN(0,I)$ be a standard Gaussian $\bZ$-valued random variable with distribution $\mu_Z$ and density $k(z)$. For each $\theta\in \Theta$ let $H_{\theta_{1,0}}:\bR^{d'}\rightarrow\bK$, $H_{\theta_{2,0}}:\bK\times \bZ\rightarrow \bR^{d}$, $G_{\theta_1}:\bR\times \bK\rightarrow L(\bR^{d'},\bK)$ and $G_{\theta_2}:\bR\times \bR^{d}\rightarrow \bR^{d}$ be continuous functions in all its variables. Let $(X,Y)_{t\in [0,T]}$ be the solution of \eqref{equ SDE}. Henceforth we fix an $s\in [0,T]$.
Consider the following generator equations, which will serve as an estimator for sample paths of $\E(X_t|\cF^Y_s)$, $t\in [s,T]$.
\begin{enumerate}
    \item For $ r\in[0,s]$, let $\tilde X_r$ be a $\mathcal F_r^Y$-adapted $\bK$-valued process given by the controlled differential equation (CDE)
\begin{equation}
\label{eq:gen1}
\begin{split}
    \tilde{X}_0 &=  H_{\theta_{1,0}}(Y_0)\\
    \tilde{X}_r &= \tilde{X}_0 + \int_0^r  G_{\theta_1}(u,\tilde{X}_u)\, dY_u.
\end{split}
\end{equation}
\item For $t\in [s,T]$, let with $z\sim\cN(0,I)$,
\begin{equation}
\label{eq:gen2}
\begin{split}
    X_s^z = & H_{\theta_{2,0}}(\tilde{X}_s,z), \\
    X_r^z = & X_s^z + \int_s^r G_{\theta_2}(u,X_s^z)\, du.
\end{split}
\end{equation}
\end{enumerate}

Equation \eqref{eq:gen1} is solved using the Log-ODE method \cite{morrill2021neural}, and $H_{\theta_{1,0}},  G_{\theta_1}$ are parametrized feed-forward NNs, with $\theta_{1,0}$ and $\theta_1$ denoting the learning parameters. Similarly, $H_{\theta_{2,0}}$ and $G_{\theta_2}$ are feed-forward neural nets, parametrized by $\theta_{2,0}$ and $\theta_2$ respectively, where equation \eqref{eq:gen2} can be solved by any ODE solver. The aim is to find $\theta\in\Theta$, such that for each $t\in [s,T]$  we have
\begin{equation}
\label{eq:sample}
\E(X_t|\cF^Y_s)\approx \frac{1}{N}\sum_{i=1}^N X^{z_i}_t,
\end{equation}
where $z_i, i=1,\dots,N$ are samples from the random variable $z$. More precisely, if the mappings $H_{\theta_{1,0}},H_{\theta_{2,0}},G_{\theta_1}$ and $G_{\theta_2}$ are such that for $t\in [s,T]$,
\begin{equation*}
    \E (X_t|\cF^Y_s) = \int_\bZ X^z_t \,\mu_Z(dz),
\end{equation*}
then \eqref{eq:sample} is an example of simple random sampling \cite{AG2007}. It is known that then, by the Law of Large Numbers (LLN), the right-hand side of \eqref{eq:sample} converges to $\E(\vp(X_t)|\cF^Y_s)$ almost surely as the sample size $N\rightarrow \infty$. 
In the remainder of this subsection we show that the model \eqref{eq:gen1}-\eqref{eq:gen2} is well posed. In other words, we argue that the approximation error of \eqref{eq:sample} can be made arbitrarily small by the right choice of $\theta\in\Theta$.

Recall the notation $\bar{X}$ for the time-augmented path $(t,X_t)_{t\in [0,T]}$. Let $m,k\in \bN$. The following result, Theorem  B.7 in \cite{KMFL2020}, shows that CDEs of the form \eqref{eq:gen1} are universal approximators in $\cV^1$.
\begin{lemma}
\label{lm:universal approximator}
Let a path $R\in\cV^1([0,T],\bR^k)$. Then for any $\varepsilon>0$ there exist continuous functions $f_0:\bR^d\rightarrow\bR^m$ and $f:\bR^m\rightarrow\bR^{m\times(d+1)}$, a linear map $l:\bR^m\rightarrow\bR^k$ and a path $X\in \cV^1([0,T],\bR^d)$ such that the unique solution of the CDE
$$
\tilde{R}_t = \tilde{R}_0+\int_0^t f(\tilde{R}_r)\,d\bar{X}_r,\quad \tilde{R}_0 = f_0(X_0),
$$
satisfies $\|R-l(\tilde{R})\|_{\cV^1}\leq \varepsilon.$
\end{lemma}
It is known, \cite{RL2018}, that for some $\cF^Y_r$-predictable process $C$ we can write
\begin{equation}
\label{martingale rep}
\E(X_r|\cF^Y_r) = \E(X_0|\cF^Y_0) + \int_0^r C_u\,dY_u.
\end{equation}
Consider, for $\cD$, the processes $\widehat{\E(X_r|\cF^Y_r)}$, $r\in [0,s]$ and $\hat{Y}$.
Then Lemma \ref{lm:universal approximator} together with the martingale representation \eqref{martingale rep} ensures that \eqref{eq:gen1}, for the right choice of $\theta\in\Theta$ and with $\hat{Y}$ instead of $Y$, is an efficient and accurate estimator for $\widehat{\E(X_r|\cF^Y_r)}$, $r\in [0,s]$. 
A version of the universal approximator property, Lemma  \ref{lm:universal approximator}, can also be proven for sample paths of $(\E(X_t|\cF^Y_t))_{t\in [0,T]}$ with $Y$ as driver for the CDE.
A rigorous proof exceeds the scope of the present article and will be given in a follow-up paper.

The following result gives an explicit form to the mappings $H_{\theta_{2,0}}$ and $G_{\theta_2}$ in \eqref{eq:gen2}, where we suppress the dependence on $\theta_{2,0},\theta_2$ and $\tilde{X}$ for the reader's convenience. In other words, we consider model \eqref{eq:gen1}-\eqref{eq:gen2} for a fixed $\omega\in\Omega$. Recall that $z\sim \cN(0,1)$ is a standard Gaussian $\bZ=\bR^{d}$-valued random variable with distribution $\mu_Z$. In a continuation of the present article we will present a more general class of diffeomorphic generators.

\begin{lemma}
\label{lm:diffeomorphisms}
Assume the transition probability $p(t,x,y;s,x',y')$ from Theorem \ref{th:prediction} is continuously differentiable in time for all $s,t\in [0,T]$ and $(x,y),(x',y')\in \bR^{d+d'}$. Fix $\omega\in\Omega$. \newline
(i) There exists a diffeomorphism $H:\bZ\rightarrow \bR^{d}$ such that $(\mu\circ H^{-1}) = \pi_s$. Then
\begin{equation}
\label{eq:filter}
    \E(X_s|\cF^Y_s) = \int_{\bZ} H(z)\,\mu_Z(dz).
\end{equation}
(ii) Moreover, the mean prospective transition
\begin{equation}
\label{eq:generator 2}
G(t,x'):=\frac{\partial}{\partial t}\int_{\bR^{d}}x\int_{\bR^{d'}}p(t,x,y;s,x',Y_s)\,dy\,dx
\end{equation}
satisfies
\begin{equation}
\label{eq:predictor}
\E(X_t|\cF^Y_s)=\int_{\bZ}H(z)\,\mu(dz) +  \int_{\bZ}\int_s^tG(r,H(z))\,dr\,\mu_Z(dz).
\end{equation}
\end{lemma}
\begin{proof}
Let $A$ and $B$ denote the first and second term on the right-hand side of \eqref{eq:predictor} respectively. It is known, see \cite{V2009}, that since $\mu_Z$ and $\pi_s$ are absolutely continuous with respect to the Lebesgue measure,
 there is a diffeomorphism $H:\bZ\rightarrow\bR^{d}$ such that $(\mu_z\circ H^{-1})(dx) = P(X_s\in dx|\cF^Y_s)$, or alternatively, \cite[Ch. 1]{V2009},
$$
k(z) = \pi_s(H(z))|\det DH(z)|.
$$
Then, a calculation using change of variables yields, for $\vp\in C_0^\infty(\bR^d)$,
$$
A=\int_{\bZ}\vp(H(z)) k(z)\,dz 
$$
$$
= \int_{\bZ}\vp(H(z)) \pi_s(H(z))|\det DH(z)|\, dz
$$
$$
 = \int_{\bR^{d}}\vp(x) \pi_s(x)\,dx = \E(\vp(X_s)|\cF^Y_s).
$$
This proves \eqref{eq:filter}.
For the second term, we compute
$$
B=\int_{\bZ}\int_s^t G(r,H(z))k(z)\,dr\,dz
$$
$$
=\int_{\bZ}\int_s^t G(r,H(z))\pi_s(H(z))|\det DH(z)|\,dr\,dz
$$
$$
 = \int_{\bR^{d}}\int_s^t G(r,x')\pi_s(x')\,dr\,dx'
 $$
 $$
 = \int_{\bR^{d}}x\pi_{t,s}(x)\,dx -  \int_{\bR^{d}}x\pi_{s,s}(x)\,dx = \E(X_t|\cF^Y_s) - \E(X_s|\cF^Y_s),
$$
where we used the form \eqref{eq:generator 2}. This finishes the proof.
\end{proof}
In applications the data corresponds to a collection of values
\begin{equation}
\label{eq:data}
\bD_{\cD,m}:=
\{(X_{t_i},Y_{t_i})(\omega_j),t_i\in \cD,  \omega_j\in\Omega, j=1,\dots,m\},
\end{equation}
where $|\cD|=n$, for integers $n,m\in\bN$. The estimator \eqref{eq:gen1}-\eqref{eq:gen2} is then  applied to the interpolated paths $\{(\hat{X}(\omega_j),\hat{Y}(\omega_j)),j=1,\dots,m\}$. Depending on the context and required regularity, different interpolation techniques can be used \cite{KMFL2020,marc-sig}. In the context of our well-posedness results, note that (i) always $(\hat{X},\hat{Y})\in \cV^1$ and that (ii), we only require a finite collection of diffeomorphisms $H(\omega_j,\cdot)$, which is sure to exist. This motivates the following.
\begin{claim}
For each $\varepsilon>0$ and $\delta>0$ there exists
an equidistant partition $\cD$ of size $m$, a data set $\bD_{{\cD,m}}$ \eqref{eq:data} and   mappings $H_{\theta_{1,0}}, G_{\theta_1}, H_{\theta_{2,0}}, G_{\theta_2}$ such that
\begin{equation}
\label{eq:convergence}
P\left(\big|\E(X_t|\cF^Y_s)-\tfrac{1}{N}\sum_{i-1}^N X^{z_i}_t\big|\geq\delta\right)\leq\varepsilon.
\end{equation}
\end{claim}
\begin{proof}
The proof of this result is quite technical and will be the subject of a subsequent paper. Here we only provide a rough sketch. For simplicity of exposition we set $d=d'=k=1$ and only present the case $t=s$, corresponding to the filtering problem.  First observe that $\sup_{t\in [0,T]}|X_t|$ is uniformly integrable, meaning that for each $\varepsilon$ there exists a $K_\varepsilon$ such that $P(\sup_{t\in [0,T]}|X_t|\geq K_\varepsilon)\leq \varepsilon$. Next we partition the set $\Omega^\varepsilon:=[\sup_{t\in [0,T]}|X_t|<K_\varepsilon]$ of probability $1-\varepsilon$. For $r_i\in \cD\cap [0,s]$, $i=1,\dots,n'$ and numbers $k_j=\tfrac{K_\varepsilon}{M}j$, $j=0,\dots,M$, assume without loss of generality that $X_t\geq 0$, $t\in [0,s]$, and define sets of the form
$$
A^{i}_j:=\{\omega\in\Omega^\varepsilon: \E(X_{r_i}|\cF^Y_{r_i})\in [k_j,k_{j+1}]\},
$$
$$
B_{\bJ}:=\bigcap_{i=1,\dots,n';j_i\in\{0,\dots,N\}}A^i_{k_{j_i}}.
$$
The set $B_{\bJ}$ prescribes a certain range of paths for $\E(X_r|\cF^Y_r)(\omega)$, when $\omega\in B_\bJ$. We show that due to the integrability and continuity of $(X,Y)$ and the universal approximation property of Lemma \ref{lm:universal approximator} we can obtain \eqref{eq:convergence} on $B_\bJ$. Repeating this procedure yields the result for $t=s$.
\end{proof}

Recall the notation for $C_\cD$. If $\cD$ consists of $n$ points, then the identification $C_\cD\ni f \longleftrightarrow (f_{t_i})_{t_i\in \cD}\in \bR^{d\times n}$ induces a $\sigma$-algebra on $C_\cD$ in a natural way. Then, as for each $\omega\in\Omega$, $\hat{X}(\omega)\in C_\cD$, by an abuse of notation we can equivalently regard $\hat{X}$ as $C_\cD$-valued random variable with a distribution $\mu^{\hat{X}}$ on $C_\cD$. Further, by 
 Corollary \ref{cor:signature uniqueness}, we have a bijection of the signature map
$$
S:C_\cD\rightarrow S(C_\cD),
$$
enabling us to consider the push-forward measure $\mu^{\hat{X}}\circ S^{-1}$ on $S(C_\cD)$. Let $\cP(S(C_\cD))$ denote the set of all such push-forward measures.
The following result, Proposition 1 and Corollary 3.3 in \cite{CO2018}, is fundamental to our methods.

\begin{lemma}
Consider two continuous $\bR^d$-valued processes $(X_t)_{t\in [0,T]}$ and $(Y_t)_{t \in [0,T]}$ such that $\sup_{t\in [0,T]}\E(|X_t|+|Y_t|)<\infty$. Denote by $\mu^{\hat{X}},\mu^{\hat{Y}}\in \cP(C_\cD)$ the push-forward measures constructed above and denote by $\E_{\mu^{\hat{X}}},\E_{\mu^{\hat{Y}}}$ the expectation under $\mu^{\hat{X}},\mu^{\hat{Y}}$ respectively. Then 
\begin{equation}
\label{eq:measure equation}
\E_{\mu^{\hat{X}}}(S)
=\E_{\mu^{\hat{Y}}}(S)\quad\text{iff} \quad\mu^{\hat{X}}=\mu^{\hat{Y}}.
\end{equation}
Then naturally $\hat{X}=\hat{Y}$ in distribution.
\end{lemma}

As we want to learn the conditional law $P(X_t\in dx|\cF^Y_s)$ for the solution $(X,Y)$ of our diffusion system \eqref{equ SDE} based on a finite set of data points at times in $\cD$, we need a metric on $\cP(S(C_\cD))$ which makes use of \eqref{eq:measure equation}. Making use of Theorem \ref{th:functionals} allows us to approximate the usual Wasserstein metric by the Sig-$W_1$ metric
\begin{equation}
\label{eq:wasserstein sig}
\begin{split}
W_1(\mu,\nu) &= \sup_{\|f\|_{Lip,1}\leq 1}\E_\mu(f(S))-\E_\nu(f(S))\\
&\approx \sup_{\|\bL\|_{Lip,1}\leq 1, \bL\text{ is linear}}\E_\mu(\bL S)-\E_\nu(\bL S)\\
&:= \text{Sig-}W_1(\mu,\nu)
\end{split}
\end{equation}
where $\mu,\nu\in\cP(S(C_\cD))$. In a very useful way, if $\mu,\nu\in\cP(S(C_\cD))$ have compact support we get\footnote{Equation \eqref{eq:wasserstein sig2} provides an explicit form of the supremum in \eqref{eq:wasserstein sig}. This is not the case in similar settings, where $f$ is parametrized by an NN, and a min-max problem is numerically solved by alternating gradient descent and gradient ascent algorithms, with additional constraints to ensure Lipschitzness. It is well known that first order gradient descent/ascent might not converge even in the convex-concave case~\cite{daskalakis2018limit}.}
\begin{equation}\label{eq:wasserstein sig2}
\begin{split}
\text{Sig-}W_1(\mu,\nu)= \|\E_\mu(S)-\E_\nu(S)\|_2,
\end{split}
\end{equation}
where the subscript $2$ denotes the $L_2$-norm on the signature space. In practice the truncated signature is used, which results in the use of the Euclidian norm.
 For a detailed derivation see the recent work \cite{marc-sig}, \cite{ni2020conditional}, where this is first introduced.
This justifies the use of signatures in conditional generative adversarial networks, resulting in the \textit{Conditional Sig-Wasserstein GAN} (CSigWGAN).

\subsection{Implementation}

\textbf{Training.} Consider again the model \eqref{eq:gen1}-\eqref{eq:gen2} and the setup in subsection \ref{subsec: estimator}. To train the estimator we generate data $\bD_{\cD,m}$ as in \eqref{eq:data}. Considering $\widehat{\E(X|\cF^Y)}$ as random variable on the space $C_\cD$, we obtain an approximative measure $\mu_{\bD}$, which for each fixed time $t>s$ approximates the prediction measure $P(X_t\in dx|\cF^Y_s)$. Then, using the Sig-$W_1$- metric in \eqref{eq:wasserstein sig}, we train the neural nets in the estimator \eqref{eq:gen1}-\eqref{eq:gen2} so that  
\begin{equation}\label{eq:wasserstein1}
    \theta^* = \argmin_{\theta}\,\, \mathbb E(W_1(\mu_{\bD}, \nu_{\theta})),
\end{equation}
where $\nu_\theta$ is the distribution of the approximated conditional expectation in \eqref{eq:sample}, i.e. our estimator. This is outlined in Algorithm \ref{alg:CSigWGAN}.

\begin{algorithm}
\caption{Training and evaluation of CSigWGAN}
\label{alg:CSigWGAN}
\begin{algorithmic}
\STATE{\textbf{Input}: i) Time discretisation $\cD:=\{0=t_0<\ldots<t_N=T\}$ and fixed $s\leq t\in\cD$, \newline
ii) Training dataset $\bD_{\cD,m} := \{(X_{t_i}, Y_{t_i})(\omega_j), t_i\in \cD; j=1,\ldots,m\}$}
\STATE{\textbf{Notation}: $\mathbb E^{\bD}$ denotes the empirical expectation calculated on the dataset $\bD$, }
\STATE{ \textbf{Training}:
\begin{enumerate}
    \item Approximate the conditional expectation
    $\mathbb E(\mathbf{X}_{s,t}^N | \mathcal F_s^Y)$
     under the data measure by the $L_2$-orthogonal projection of $\mathbf{X}_{s,t}^N$ on the space of  $\mathcal F_s^Y$-measurable r.v., by leveraging Doob-Dynkin lemma and Theorem \ref{th:functionals},
    \[
    \begin{split}
    \hat{L} := \argmin_{L\text{ is linear}} \mathbb E^{\bD} \left[(\mathbf{X}_{s,t}^N - L(\mathbf{Y}_{0,s}^N))^2\right], \\
    \mathbb E(\mathbf{X}_{s,t}^N | \mathcal F_s^Y)(\omega) \approx 
    \hat{L}(\mathbf{Y}_{0,s}^N(\omega)).
    \end{split}
    \]
    \item Use Stochastic Gradient Descent to minimise \eqref{eq:wasserstein sig2},
    \[
    \theta^* = \argmin_{\theta} \mathbb E^{\bD} \left[ \lVert \hat{L}(\mathbf{Y}_{0,s}^N) - \mathbb E_{\nu_{\theta}}\left[\mathbf X_{s,t}^N\right] \rVert_2 \right],
    \]
    where $\mathbb E_{\nu_{\theta}}\left[\mathbf X_{s,t}^N\right]$ can be estimated using Monte Carlo by drawing samples from the generator using different values of $z$ in the generators \eqref{eq:gen1}-\eqref{eq:gen2}.
\end{enumerate} }
\RETURN $\theta^{*}$.
\end{algorithmic}
\end{algorithm}

\section{Numerical results}
\begin{figure*}[!ht]
	\centering
		\includegraphics[width=\textwidth]{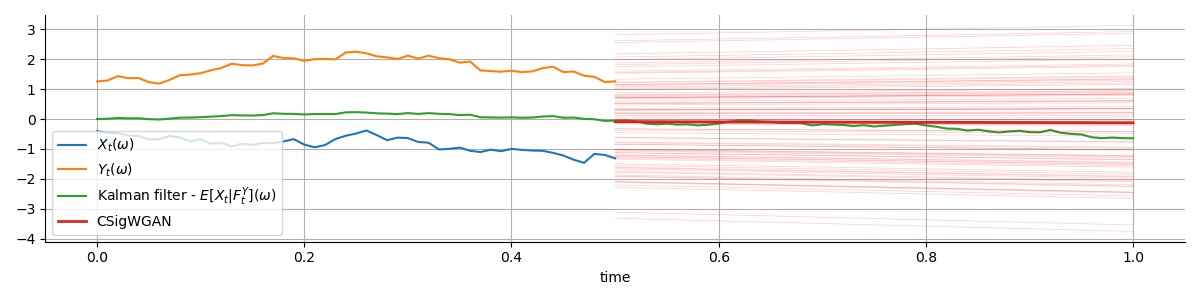}
		\caption{CSigWGAN compared to the Kalman filter: The average over the sample paths generated by the CSigWGAN (bold red line) remains very close to the Kalman filter (green) for $t>s$, and matches the Kalman filter at $t=s$. The light red lines are individual sample paths, sampled from the conditional prediction law.}
\label{fig:sigfiltering}
\end{figure*}

We consider the linear SDE with $X_t, Y_t\in\mathbb R$, $t\in [0,1]$
\begin{equation}\label{eq:filtering experiment}
\begin{split}
    dX_t = & 0.1(1+t)X_t dt + dV_t, \quad X_0\sim\cN(0,1) \\
    dY_t = & 0.2 X_t dt + dW_t, \quad \quad \quad Y_0\sim\cN(0,1),
\end{split}
\end{equation}
where $X_0$, $Y_0$, $V$ and $W$ are pairwise independent.
The generator~\eqref{eq:gen1}-\eqref{eq:gen2} is parametrized as follows: $h_{\theta_{1,0}}, H_{\theta_{2,0}}$ are feedforward NNs with one hidden layer with 20 neurons and \texttt{ReLU} activation function. The resulting process $\tilde X_r$ from equation~\eqref{eq:gen1} takes values in $\mathbb R^{10}$, and $X_r^z$ takes values in $\mathbb R$. The vector fields $G_{\theta_1}, G_{\theta_2}$ are parametrized by feedforward NNs with one hidden layer with 128 hidden neurons and \texttt{Tanh} activation function.

We use~\eqref{eq:filtering experiment} to create a synthetic dataset~\eqref{eq:data} with 20 000 samples. The algorithm is trained for a total of 50 epochs, where the NDEs are backpropagated using the adjoint method~\cite{chen2018neural}. 

Fig.~\ref{fig:sigfiltering} provides an example of the numerical approximation of $\mathbb E[X_t | \mathcal F_s^Y], t\geq s, s=0.5$ compared to the existing analytical solution of $\mathbb E[X_t | \mathcal F_t^Y]$ given by the Kalman filter. Each red line is a sample generated with~\eqref{eq:gen1}-\eqref{eq:gen2} for one sample of $z$. The conditional expectation, depicted in stronger red, is then approximated using the average of the samples according to \eqref{eq:sample}. For $t=s$, our estimator matches the Kalman filter. 

Prediction methods are used less frequently in the literature in comparison to filtering and are more challenging to implement. However, the Kalman filter, being the optimal estimate and using more information on the observation process $Y$, can be used well for validation purposes.

The implementation is available at \url{https://github.com/msabvid/SigFiltering}. 

\bibliographystyle{IEEEtran}
\bibliography{bibliography}

\begin{thebibliography}{10}
\providecommand{\url}[1]{#1}
\csname url@samestyle\endcsname
\providecommand{\newblock}{\relax}
\providecommand{\bibinfo}[2]{#2}
\providecommand{\BIBentrySTDinterwordspacing}{\spaceskip=0pt\relax}
\providecommand{\BIBentryALTinterwordstretchfactor}{4}
\providecommand{\BIBentryALTinterwordspacing}{\spaceskip=\fontdimen2\font plus
\BIBentryALTinterwordstretchfactor\fontdimen3\font minus
  \fontdimen4\font\relax}
\providecommand{\BIBforeignlanguage}[2]{{%
\expandafter\ifx\csname l@#1\endcsname\relax
\typeout{** WARNING: IEEEtran.bst: No hyphenation pattern has been}%
\typeout{** loaded for the language `#1'. Using the pattern for}%
\typeout{** the default language instead.}%
\else
\language=\csname l@#1\endcsname
\fi
#2}}
\providecommand{\BIBdecl}{\relax}
\BIBdecl

\bibitem{C2014}
D.~Crisan, ``The stochastic filtering problem: a brief historical account,''
  \emph{Journal of Applied Probability}, vol.~51, no.~A, pp. 13--22, 2014.

\bibitem{BC2009}
A.~Bain and D.~Crisan, \emph{Fundamentals of stochastic filtering}.\hskip 1em
  plus 0.5em minus 0.4em\relax Springer, 2009, vol.~3.

\bibitem{CCC2014}
T.~Cass, M.~Clark, and D.~Crisan, ``The filtering equations revisited,'' in
  \emph{Stochastic Analysis and Applications 2014}.\hskip 1em plus 0.5em minus
  0.4em\relax Springer, 2014, pp. 129--162.

\bibitem{RL2018}
B.~L. Rozovsky and S.~V. Lototsky, \emph{Stochastic evolution systems: linear
  theory and applications to non-linear filtering}.\hskip 1em plus 0.5em minus
  0.4em\relax Springer, 2018, vol.~89.

\bibitem{Y1977}
M.~Yor, ``Sur les th{\'e}ories du filtrage et de la pr{\'e}diction,'' in
  \emph{S{\'e}minaire de Probabilit{\'e}s XI}.\hskip 1em plus 0.5em minus
  0.4em\relax Springer, 1977, pp. 257--297.

\bibitem{KR1978}
N.~V. Krylov and B.~Rozovskii, ``On conditional distributions of diffusion
  processes,'' \emph{Mathematics of the USSR-Izvestiya}, vol.~12, no.~2, p.
  336, 1978.

\bibitem{R1980}
B.~L. Rozovskii, ``On conditional distributions of degenerate diffusion
  processes,'' \emph{Theory of Probability and Its Applications}, vol.~25,
  no.~1, pp. 147--151, 1980.

\bibitem{R1984}
B.~Rozovskii, ``Filtering, smoothing and prediction of degenerate diffusion
  processes. backward equations,'' \emph{Theory of Probability and Its
  Applications}, vol.~28, no.~4, pp. 762--774, 1984.

\bibitem{GG1}
F.~Germ and I.~Gy{\"o}ngy, ``On partially observed jump diffusions {I}. {T}he
  filtering equations,'' \emph{arXiv:2205.08286}, 2022.

\bibitem{GG2}
------, ``On partially observed jump diffusions {I}{I}. {T}he filtering
  density,'' \emph{arXiv:2205.14534}, 2022.

\bibitem{GK2003}
I.~Gy{\"o}ngy and N.~Krylov, ``On the splitting-up method and stochastic
  partial differential equations,'' \emph{The Annals of Probability}, vol.~31,
  no.~2, pp. 564--591, 2003.

\bibitem{GK2003-2}
------, ``On the rate of convergence of splitting-up approximations for
  spdes,'' in \emph{Stochastic inequalities and applications}.\hskip 1em plus
  0.5em minus 0.4em\relax Springer, 2003, pp. 301--321.

\bibitem{IR2000}
K.~Ito and B.~Rozovskii, ``Approximation of the {K}ushner equation for
  nonlinear filtering,'' \emph{SIAM Journal on Control and Optimization},
  vol.~38, no.~3, pp. 893--915, 2000.

\bibitem{CLO2022}
D.~Crisan, A.~Lobbe, and S.~Ortiz-Latorre, ``An application of the splitting-up
  method for the computation of a neural network representation for the
  solution for the filtering equations,'' \emph{arXiv preprint
  arXiv:2201.03283}, 2022.

\bibitem{RU1999}
K.~Reif and R.~Unbehauen, ``The extended {K}alman filter as an exponential
  observer for nonlinear systems,'' \emph{IEEE Transactions on Signal
  processing}, vol.~47, no.~8, pp. 2324--2328, 1999.

\bibitem{AGM2020}
S.~Afshar, F.~Germ, and K.~Morris, ``Well-posedness of extended {K}alman filter
  equations for semilinear infinite-dimensional systems,'' in \emph{2020 59th
  IEEE Conference on Decision and Control (CDC)}.\hskip 1em plus 0.5em minus
  0.4em\relax IEEE, 2020, pp. 1210--1215.

\bibitem{AGM2022}
S.~Afshar, F.~Germ, and K.~A. Morris, ``Extended {K}alman filter based observer
  design for semilinear infinite-dimensional systems,'' \emph{arXiv preprint
  arXiv:2202.07797 (submitted to IEEE TAC)}, 2022.

\bibitem{LCL2007}
T.~J. Lyons, M.~Caruana, and T.~L{\'e}vy, \emph{Differential equations driven
  by rough paths}.\hskip 1em plus 0.5em minus 0.4em\relax Springer, 2007.

\bibitem{C1954}
K.-T. Chen, ``Iterated integrals and exponential homomorphisms,''
  \emph{Proceedings of the London Mathematical Society}, vol.~3, no.~1, pp.
  502--512, 1954.

\bibitem{C1958}
------, ``Integration of paths - a faithful representation of paths by
  noncommutative formal power series,'' \emph{Transactions of the American
  Mathematical Society}, vol.~89, no.~2, pp. 395--407, 1958.

\bibitem{HL2010}
B.~Hambly and T.~Lyons, ``Uniqueness for the signature of a path of bounded
  variation and the reduced path group,'' \emph{Annals of Mathematics}, pp.
  109--167, 2010.

\bibitem{BGLY2016}
H.~Boedihardjo, X.~Geng, T.~Lyons, and D.~Yang, ``The signature of a rough
  path: uniqueness,'' \emph{Advances in Mathematics}, vol. 293, pp. 720--737,
  2016.

\bibitem{CK2016}
I.~Chevyrev and A.~Kormilitzin, ``A primer on the signature method in machine
  learning,'' \emph{arXiv preprint arXiv:1603.03788}, 2016.

\bibitem{CO2018}
I.~Chevyrev and H.~Oberhauser, ``Signature moments to characterize laws of
  stochastic processes,'' \emph{arXiv preprint arXiv:1810.10971}, 2018.

\bibitem{goodfellow2014generative}
I.~Goodfellow, J.~Pouget-Abadie, M.~Mirza, B.~Xu, D.~Warde-Farley, S.~Ozair,
  A.~Courville, and Y.~Bengio, ``Generative adversarial nets,'' \emph{Advances
  in neural information processing systems}, vol.~27, 2014.

\bibitem{marc-sig}
H.~Ni, L.~Szpruch, M.~Sabate-Vidales, B.~Xiao, M.~Wiese, and S.~Liao,
  ``Sig-{W}asserstein {GAN}s for time series generation,'' \emph{arXiv preprint
  arXiv:2111.01207}, 2021.

\bibitem{yoon2019time}
J.~Yoon, D.~Jarrett, and M.~Van~der Schaar, ``Time-series generative
  adversarial networks,'' \emph{Advances in Neural Information Processing
  Systems}, vol.~32, 2019.

\bibitem{KMFL2020}
P.~Kidger, J.~Morrill, J.~Foster, and T.~Lyons, ``Neural controlled
  differential equations for irregular time series,'' \emph{Advances in Neural
  Information Processing Systems}, vol.~33, pp. 6696--6707, 2020.

\bibitem{chen2018neural}
R.~T. Chen, Y.~Rubanova, J.~Bettencourt, and D.~K. Duvenaud, ``Neural ordinary
  differential equations,'' \emph{Advances in neural information processing
  systems}, vol.~31, 2018.

\bibitem{lecun1988theoretical}
Y.~LeCun, D.~Touresky, G.~Hinton, and T.~Sejnowski, ``A theoretical framework
  for back-propagation,'' in \emph{Proceedings of the 1988 connectionist models
  summer school}, vol.~1, 1988, pp. 21--28.

\bibitem{pearlmutter1995gradient}
B.~A. Pearlmutter, ``Gradient calculations for dynamic recurrent neural
  networks: A survey,'' \emph{IEEE Transactions on Neural networks}, vol.~6,
  no.~5, pp. 1212--1228, 1995.

\bibitem{kidger2022neural}
P.~Kidger, ``On neural differential equations,'' \emph{arXiv preprint
  arXiv:2202.02435}, 2022.

\bibitem{LL2016}
D.~Levin, T.~Lyons, and H.~Ni, ``Learning from the past, predicting the
  statistics for the future, learning an evolving system,'' \emph{arXiv
  preprint arXiv:1309.0260v6}, 2016.

\bibitem{ni2020conditional}
H.~Ni, L.~Szpruch, M.~Wiese, S.~Liao, and B.~Xiao, ``Conditional
  sig-wasserstein gans for time series generation,'' \emph{arXiv preprint
  arXiv:2006.05421}, 2020.

\bibitem{morrill2021neural}
J.~Morrill, C.~Salvi, P.~Kidger, and J.~Foster, ``Neural rough differential
  equations for long time series,'' in \emph{International Conference on
  Machine Learning}.\hskip 1em plus 0.5em minus 0.4em\relax PMLR, 2021, pp.
  7829--7838.

\bibitem{AG2007}
S.~Asmussen and P.~W. Glynn, \emph{Stochastic simulation: algorithms and
  analysis}.\hskip 1em plus 0.5em minus 0.4em\relax Springer, 2007, vol.~57.

\bibitem{V2009}
C.~Villani, \emph{Optimal transport: old and new}.\hskip 1em plus 0.5em minus
  0.4em\relax Springer, 2009, vol. 338.

\bibitem{daskalakis2018limit}
C.~Daskalakis and I.~Panageas, ``The limit points of (optimistic) gradient
  descent in min-max optimization,'' \emph{Advances in Neural Information
  Processing Systems}, vol.~31, 2018.

\end{thebibliography}

\end{document}